\pdfoutput=1
\documentclass[sigconf]{acmart}

\usepackage{subcaption} 
\usepackage[linesnumbered,ruled,vlined]{algorithm2e}

\newtheorem{assumption}{Assumption}

\newtheorem{theorem}{Theorem}

\allowdisplaybreaks
\DeclareMathOperator*{\argmin}{arg\,min}

\copyrightyear{2026}
\acmYear{2026}
\setcopyright{rightsretained}
\acmConference[MobiHoc '26]{The Twenty-seventh International Symposium on Theory, Algorithmic Foundations, and Protocol Design for Mobile Networks and Mobile Computing}{November 23--26, 2026}{Tokyo, Japan}
\acmBooktitle{The Twenty-seventh International Symposium on Theory, Algorithmic Foundations, and Protocol Design for Mobile Networks and Mobile Computing (MobiHoc '26), November 23--26, 2026, Tokyo, Japan}
\acmDOI{10.1145/3842721.3849640}
\acmISBN{979-8-4007-2975-1/2026/11}


\begin{document}

\title{Foundation Model-Aided Multi-Agent Reinforcement Learning for Wireless Random Access Network Optimization}

\author{Myeung Suk Oh}
\email{oh.746@osu.edu}
\affiliation{
  \institution{The Ohio State University}
  \city{Columbus}
  \state{Ohio}
  \country{USA}
}

\author{Zhiyao Zhang}
\email{zhang.15178@osu.edu}
\affiliation{
  \institution{The Ohio State University}
  \city{Columbus}
  \state{Ohio}
  \country{USA}
}

\author{Alvaro Velasquez}
\email{alvaro.velasquez@colorado.edu}
\affiliation{
  \institution{University of Colorado Boulder}
  \city{Boulder}
  \state{Colorado}
  \country{USA}
}

\author{Nathaniel D. Bastian }
\email{nbastian@pathfinderappliedresearch.com}
\affiliation{
  \institution{Pathfinder Applied Research, LLC}
  \city{Ashburn}
  \state{Virginia}
  \country{USA}
}

\author{Jia Liu}
\email{liu@ece.osu.edu}
\affiliation{
  \institution{The Ohio State University}
  \city{Columbus}
  \state{Ohio}
  \country{USA}
}

\begin{abstract}
    Random access (RA) is one of the most foundational medium access control (MAC) layer scheduling schemes for handling unpredictable data traffic from multiple terminals. While multi-agent reinforcement learning (MARL) has been explored to optimize RA-based wireless networks, its reliance on experience-driven, distributed policy learning incurs significant training overhead for each optimization task, limiting its feasibility in real-world applications. In this work, we propose to leverage a foundation model (FM) to improve MARL efficiency across diverse RA network optimization tasks. Specifically, we design an FM-aided actor-critic algorithm within a consensus-based decentralized MARL architecture and provide its convergence analysis under local reward exchanges and nonlinear value function approximations to show that our algorithm achieves the same convergence order as the conventional MARL with critic model exchanges and linear approximations. Our numerical results show that our FM-based approach significantly enhances MARL speed for RA network optimization.
\end{abstract}


\maketitle

\section{Introduction}\label{sec:intro}

{\bf 1) Background and Motivations:}
Random Access (RA) techniques at the Medium Access Control (MAC) layer are fundamental to modern wireless standards, including Wi-Fi and LTE Licensed Assisted Access~\cite{Chen16}.
In RA-based networks (Figure~\ref{fig:system_diagram}), multiple devices contend for a shared communication channel rather than relying on centralized scheduling.
By allowing devices to independently determine transmission opportunities, RA eliminates the overhead associated with pre-allocated channel access.
This decentralized paradigm is particularly well suited for the Internet of Things (IoT)~\cite{Aouedi24}, machine-type communications (MTC)~\cite{Mahmood21}, and smart grids~\cite{Rodriguez23}, where low-complexity and scalable mechanisms are required to manage highly sporadic and unpredictable traffic.

Despite these advantages, the primary challenge in RA is maximizing network efficiency while mitigating collisions.
Excessive contention leads to wasted communication resources, often diminishing the benefits of decentralized access.
Although extensive research has sought to improve the performance of RA-based networks, existing approaches either rely on heuristic methods without performance guarantees or adopt idealized analytical models that struggle to capture the complexities of real-world environments.

\begin{figure}[!t]
    \centering
    \includegraphics[width=0.9\linewidth]{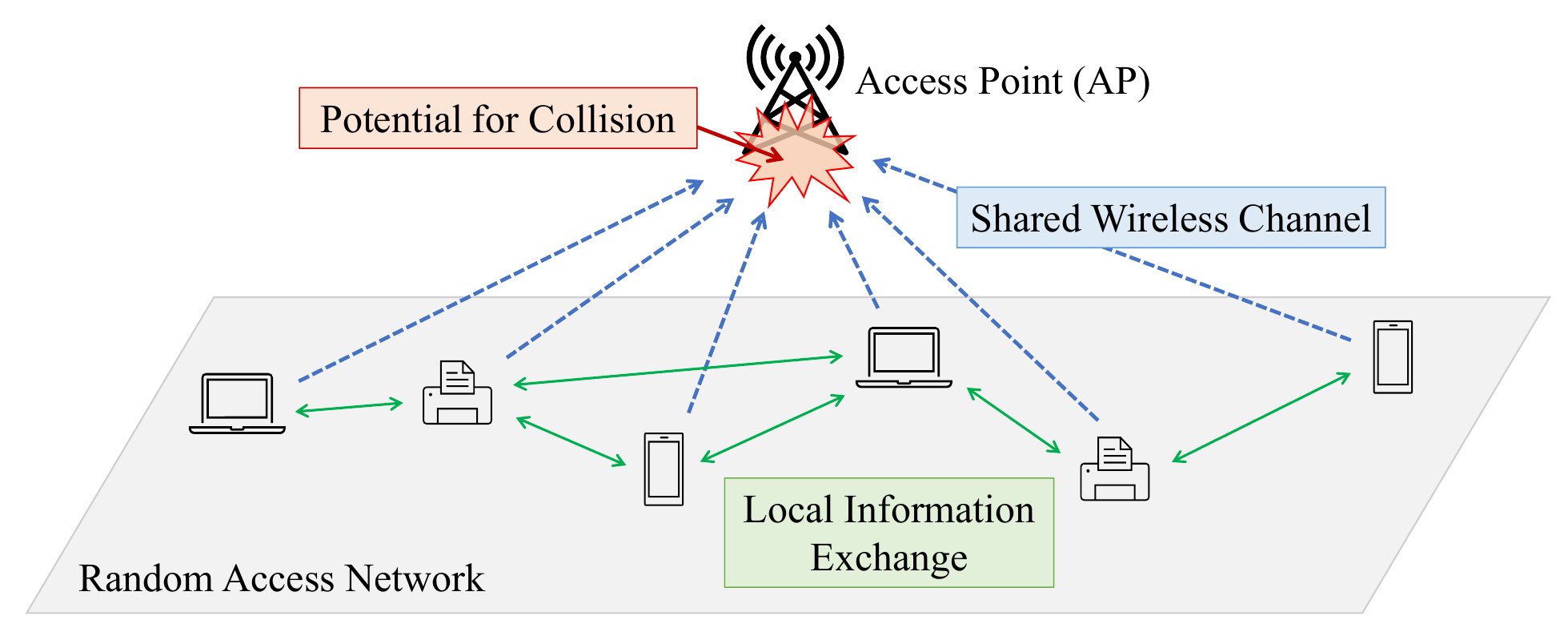}
    \vspace{-3mm}
    \caption{An RA-based wireless network, where devices contend for channel access. Each device makes its own decision on when to transmit. Communication links are established across the devices for exchanging local information.}
    \Description{.}
    \label{fig:system_diagram}
    \vspace{-5mm}
\end{figure}

\vspace{.5mm}
{\bf 2) Foundation Model-Aided Multi-Agent Reinforcement Learning:}
Recent advances in machine learning have introduced new opportunities for optimizing RA networks~\cite{Ahmad20,Kulin21,Cao21}.
Due to their distributed nature, RA networks can naturally be modeled as multi-agent reinforcement learning (MARL) systems.
Many existing approaches~\cite{Han20,Lee24,Yu19,Yu22,Zhang20,Guo22,He24} utilize the centralized training and decentralized execution (CTDE) paradigm~\cite{Amato24} to maintain learning stability.
However, CTDE-based methods face fundamental limitations.
Their reliance on a central training entity restricts deployment in large-scale or infrastructure-constrained settings.
Even when such an entity exists, communication overhead on aggregating observations and actions scales poorly with the number of devices.
Also, centralization can introduce single points of failure and raise concerns regarding privacy and adversarial vulnerability.

Furthermore, conventional MARL methods often suffer from task-specific overfitting.
Models designed for narrow objectives (e.g., collision minimization) typically require retraining or architectural modifications when adapting to new performance metrics, resulting in poor sample efficiency and limited generalization.
These limitations motivate the need for MARL frameworks that are both fully decentralized and adaptable to diverse optimization objectives.

Meanwhile, foundation models (FMs) pretrained on large-scale datasets and then finetuned for downstream tasks offer a promising path for reinforcement learning (RL).
By modeling system dynamics as sequence processes (e.g., using architectures like MAMBA~\cite{Gu23}), self-supervised FMs can learn {\em reward-agnostic} representations that allow faster policy adaptation and improved generalization across tasks.
While there is growing interest in applying FMs to wireless systems~\cite{Du24,Fontaine24,Chen24,Yang25_2}, several open challenges remain.
Particularly, it is unclear how to effectively integrate FMs into established RL frameworks (e.g., actor–critic methods) with theoretical performance guarantees, and how to design MARL algorithms that exploit the representational power of FMs without introducing centralized dependencies.
This motivates us to pursue FM-aided MARL for wireless random access network optimization.

\vspace{.5mm}
{\bf 3) Our Approach and Contributions:} 
In this work, we propose \textit{an FM-aided MARL framework} for optimizing RA-based wireless networks.
To ensure practical viability, we adopt a fully decentralized paradigm based on decentralized training and decentralized execution (DTDE).
To maintain stability and guarantee convergence without a central coordinator, we incorporate an average consensus mechanism~\cite{Boyd06}, allowing devices to synchronize via local peer-to-peer communication.
Our proposed framework employs a specialized actor-critic (AC) design, where a pretrained FM serves as \textit{a reward-agnostic backbone} for the critic network.
By attaching a {\em task-specific head} to this backbone, the critic benefits from expressive latent representations of system dynamics learned through {\em self-supervised} pretraining on reward-independent data.
This integration accelerates online policy learning while achieving performance comparable to or better than end-to-end MARL training.
Our main contributions are summarized as follows.

\begin{list}{\labelitemi}{\leftmargin=1.1em \itemindent=-0.0em \itemsep=.2em}
    \item We consider a {\em consensus-driven, fully decentralized MARL} that eliminates the need for global controllers.
    Relying on local exchanges, the design is well suited for scenarios where centralized control is infeasible or undesirable.
    Unlike~\cite{Oh2025}, we consider three RA optimization objectives, namely \textit{fair age of information (AoI)}, \textit{maximum sum-rate}, and \textit{fair rate}, and formulate unified Markov decision process (MDP) models with shared observation and action spaces, as well as locally computable reward functions.

    \item We introduce a {\em new} method for integrating FMs into the actor–critic pipeline.
    Inspired by the SMART architecture~\cite{Sun23}, we pretrain an FM for forward-dynamics prediction and incorporate it as a reward-agnostic backbone within the critic.
    The resulting latent representations significantly enhance the critic’s ability to evaluate rewards during online learning.
    
    \item Unlike \cite{Sun23}, which is purely empirical, we provide a finite-time convergence analysis of our FM-aided decentralized actor–critic algorithm.
    Our analysis addresses consensus applied to local rewards and nonlinear state-value approximation, offering {\em theoretical} insights that differentiate this work from existing decentralized MARL literature.
    Through extensive numerical simulations, we show that our framework outperforms end-to-end training in learning speed across multiple RA optimization tasks.
\end{list}

\vspace{-2.5mm}
\section{Related Work}\label{sec:work}
\vspace{-.5mm}

\textbf{1) Traditional RA Network Optimization:} RA techniques are broadly classified into two types: sensing-free and sensing-based.
Sensing-free protocols, like ALOHA~\cite{Abramson70}, are valued for their simplicity in satellite and multi-hop networks~\cite{Yu23,Baccelli06} but suffer from inherently high collision rates.
In contrast, sensing-based RA employs a listen-before-talk (LBT) mechanism.
The most prominent example, CSMA/CA~\cite{Kleinrock75}, uses random backoff to mitigate contention.
However, its heuristic nature often results in excessive delays and suboptimal channel utilization.
Moreover, its theoretical performance guarantees remain unclear.
While research into adaptive, queue-length-based CSMA~\cite{Jiang10:CSMA_TIT} shows that throughput optimality is possible, these findings often rely on idealized models that fail to maintain fairness or low latency as networks scale~\cite{Jiang10:CSMA_TON,Jiang12:CSMA_TIT,Ni12:Q-CSMA}.

\vspace{.5mm}
\textbf{2) Learning-driven RA Network Optimization:} To overcome the limitations of heuristics, recent work has explored RL for RA.
Early efforts focused on optimizing CSMA/CA parameters, such as using Q-learning or SAC-LSTM architectures to tune contention windows and waiting times~\cite{Han20,Lee24}.
However, since these methods remain dependent on the probabilistic framework of CSMA, their performance gains are fundamentally limited.

To move beyond simply parameter tuning, researchers have turned to MARL to design deterministic transmission policies.
In the literature, notable approaches include deep Q-networks for achieving $\alpha$-fairness~\cite{Yu19}, federated learning-based policy training~\cite{Zhang20}, and value factorization or policy gradient methods such as QMIX~\cite{Guo22} and MAPPO~\cite{He24}.
Despite their strong performance, these methods rely on CTDE.
The requirement for a central entity introduces significant communication overhead and raises concerns regarding scalability, data privacy, and security.
Furthermore, these models are often task-specific.
As a result, changing the optimization objective typically requires extensive architectural modification and retraining, leading to poor sample efficiency.
These challenges motivate our development of a {\em fully decentralized FM-aided} MARL framework that allows efficient wireless RA network optimization under diverse objectives without central coordination.

\vspace{.5mm}
\textbf{3) Foundation Models for Wireless Networks:} FMs have achieved state-of-the-art performance in domains such as natural language processing and computer vision through large-scale representation learning.
In wireless communications, FMs are now being explored for PHY and MAC layer tasks~\cite{Du24,Fontaine24,Chen24,Yang25_2}, especially within semantic and integrated sensing and communication (ISAC) frameworks.
In particular, their ability to generalize across heterogeneous tasks and domains makes them highly suitable for addressing the challenges posed by dynamic and complex wireless environments.
Emerging works such as WirelessGPT~\cite{Yang25_1} and patch-masked FMs~\cite{Sheng25} have demonstrated that Transformer-based architectures can effectively handle multi-task predictions, such as channel state estimation and traffic forecasting.

Similar to this line of work, we investigate the use of FMs for wireless RA network optimization.
Unlike existing RL methods that suffer from task-specific overfitting, we utilize an FM as a {\em reward-agnostic backbone} within an AC framework.
This allows for superior generalization and faster adaptation across diverse networking objectives without the need for centralized coordination.
Our approach aims to overcome the limitations of conventional RL methods that often struggle with their lack of adaptability and dependence on centralized training.

\vspace{-2mm}
\section{System Model and Problem Formulation}\label{sec:system}
\vspace{-.5mm}

\subsection{Network System and the RA Protocol}\label{ssec:system_config}

We consider an RA-based wireless network with $K$ single-antenna devices that contend for channel access to transmit data to an $N_\mathsf{r}$-antenna access point (AP).
Each device $k\in\mathcal{K}=\{1,2,\ldots,K\}$ maintains a buffer where packets are queued and transmitted in a first-come first-served manner.
We consider a saturated traffic condition, i.e., the packet arrival rate at each node is sufficiently high so that all devices consistently have data packets to transmit.

The network operates in slotted time over a finite horizon $T$, with time slots indexed by $t\in\{1,2,\ldots,T\}$.
Similar to the CSMA/CA protocol, the network adopts the LBT mechanism.
Before attempting transmission, each device first performs a clear channel assessment (CCA) to detect whether the channel is idle.
If the channel remains idle for a specified duration, the device decides whether to transmit a packet to the AP.
A device that chooses \textit{not to transmit} simply returns to the CCA phase.
If the device decides to \textit{transmit}, it sends a single data packet to the AP over $\tau$ time slots.

A transmission is considered successful if the AP correctly receives the packet without collision, in which case the AP sends out an acknowledgment (ACK) for confirmation.
This allows the device to proceed with its next packet transmission.
In the event of a collision (i.e., multiple devices have simultaneously transmitted), the AP does not return an ACK, and each involved device will wait for the next transmission opportunity to attempt retransmission.

To support fully decentralized network operation, we assume there is no global controller responsible for tasks like information aggregation or global message broadcasting.
Instead, RA devices can exchange information with their neighbors according to local communication links defined by an undirected graph $\mathcal{G}\hspace{-.5mm}=\hspace{-.5mm} (\mathcal{K},\mathcal{E})$, where $\mathcal{E}$ represents the set of device-to-device connectivity links~\cite{Boyd06}.

\vspace{-2mm}
\subsection{Data Transmission Model}\label{ssec:system_transmission}

The single-input multiple-output (SIMO) channel gain between device $k$ and the AP at time slot $t$ is modeled as
\vspace{-1mm}
\begin{equation}
    \boldsymbol{g}_{k,t} = \sqrt{\eta_{k}} \boldsymbol{h}_{k,t},
    \label{eq:device_channel}
    \vspace{-1mm}
\end{equation}
where $\boldsymbol{h}_{k,t}\in\mathbb{C}^{N_\mathsf{r}}$ represents the small-scale fading vector whose entries are independent and identically distributed (i.i.d.) complex Gaussian random variables with zero mean and unit variance.
The term $\eta_{k}$ in~\eqref{eq:device_channel} denotes the large-scale fading factor which depends on the distance between device $k$ and the AP.
For example, the magnitude of small-scale fading could follow a Rayleigh distribution, which is well-suited for modeling dense scattering wireless environments~\cite{Tse05}.
Also, our large-scale fading model can be adopted to various realistic propagation scenarios, including the ITU-R indoor~\cite{ITU} and 3GPP urban-micro (UMi)~\cite{3GPP} pathloss models.

Using~\eqref{eq:device_channel}, the achievable data rate of device $k$ at time slot $t$ can be obtained by the Shannon capacity formula:
\vspace{-1mm}
\begin{equation}
    R_{k,t} = B\log_2\left(1+\frac{p_{k}\|\boldsymbol{g}_{k,t}\|^2_2}{BN_0}\right),
    \label{eq:device_rate}
    \vspace{-1mm}
\end{equation}
where $B$ is the transmission bandwidth, $p_{k}$ is the transmit power of device $k$, and $N_0$ is the noise spectral density~\cite{Tse05}.
Assuming a block-fading channel (i.e., the channel remains static during each packet transmission), we express the amount of data which device $k$ can successfully transmit at time slot $t$ as
\vspace{-1mm}
\begin{equation}
    m_{k,t} = R_{k,t}T_\mathsf{s}\tau,
    \label{eq:device_data}
    \vspace{-1mm}
\end{equation}
where $T_\mathsf{s}$ denotes the slot duration.
Note that \eqref{eq:device_data} plays a crucial role in transmission decision-making.
For example, a device with low $m_{k,t}$ may lead to suboptimal throughput by potentially blocking other devices with more favorable channel conditions (i.e., higher $m_{k,t}$).
Since $m_{k,t}$ is directly proportional to $R_{k,t}$, each RA device needs to carefully consider its channel quality when making transmission decisions to maximize the overall network performance.

\vspace{-2mm}
\subsection{RA Network Optimization Tasks}\label{ssec:system_objectives}

To illustrate the FM's capability in supporting multiple performance goals in wireless RA network optimization with MARL, in this work, we consider the following three objectives as downstream tasks:

\vspace{.5mm}
\textbf{1) Fair-AoI (Worst-case AoI Minimization):}
AoI (age of information) measures the freshness of information by tracking the time elapsed since the latest data transmission of a device.
Let $\ell_{k,t}$ denote the number of time slots that have elapsed since device $k$'s last successful packet transmission by time slot $t$~\cite{Guo22}.
This quantity can be interpreted as the anticipated AoI if the AP successfully receives a packet from device $k$ at time slot $t$.
Let $\mathcal{T}_{k,t} \subseteq \{1,2,\ldots,t\}$ denote the set of time slot indices at which device $k$ has successfully transmitted a packet during the time horizon $t\leq T$.
Then, we minimize the worst-case AoI among all devices using
\vspace{-1mm}
\begin{equation}
    \min\max_{k\in\mathcal{K},t\in\mathcal{T}_{k,T}} \ell_{k,t}.
    \label{eq:network_AoI}
    \vspace{-1mm}
\end{equation}
Note that this mini-max formulation solely focuses on fairness in transmit frequency and does not account for channel quality or data rates.
This task is particularly useful in applications such as real-time system control in wireless drone networks, where timely information updates are important.

\vspace{.5mm}
\textbf{2) Max-Rate (Sum-rate Maximization with Minimum Rate Guarantees):}
As a classical network performance metric, sum-rate quantifies the total amount of data successfully received by the AP for a given time window.
Here, we let
\vspace{-1mm}
\begin{equation}
    x_{k,t} = \frac{1}{tT_\mathsf{s}}\sum_{t'\in\mathcal{T}_{k,t}}m_{k,t'}
    \label{eq:tput}
    \vspace{-1mm}
\end{equation}
denote the throughput of device $k$ over the time horizon $t\leq T$.
To maximize the sum-rate, i.e., $\sum_{k\in\mathcal{K}} x_{k,T}$, it is ideal for a network to prioritize RA devices with higher data rates $m_{k,t}$.
However, such a policy can result in poor fairness as it may allow devices with superior channel quality to dominate transmission resources.
To prevent this, we impose a minimum rate requirement $R_\mathsf{min}$ and formulate our objective as
\vspace{-1mm}
\begin{align}
    &\max \hspace{1mm} \sum_{k\in\mathcal{K}}x_{k,T} \label{eq:network_sum_tput} \\
    \text{s.t. } & x_{k,T} \geq R_\mathsf{min}, \hspace{1mm} \forall k\in\mathcal{K}. \label{eq:network_sum_tput1}
    \vspace{-1mm}
\end{align}
The problem above seeks to maximize the overall throughput while ensuring that each device's rate maintains at least $R_\mathsf{min}$.

\smallskip
\textbf{3) Fair-Rate (Worst-case Rate Maximization):}
In certain applications such as collaborative sensing, the system performance can be compromised by the device with the poorest throughput. In such cases, it is crucial to promote throughput fairness to prevent any single device from becoming a bottleneck.
To ensure balanced per-device rate, we maximize the minimum throughput across all devices over the time horizon $T$, which can be expressed as 
\vspace{-1mm}
\begin{align}
    \max \hspace{1mm} \min_{k\in\mathcal{K}}\,x_{k,T}.
    \label{eq:network_equal_tput}
    \vspace{-1mm}
\end{align}
This formulation ensures that no device is significantly disadvantaged in terms of data rate, thus supporting balanced and reliable system performance.

\vspace{-1mm}
\section{Foundation Model-Aided Decentralized MARL for Wireless RA Network Optimization}\label{sec:framework}

In this section, we first review the AC framework for MARL.
We then formulate the MDP tailored to our wireless RA network optimization.
Next, we describe the architecture of our FM, which serves as a reward-agnostic backbone within the AC framework.
Finally, we present the proposed FM-aided decentralized MARL algorithm and analyze its convergence properties.

\vspace{-2mm}
\subsection{Actor-Critic Framework for MARL Problem}\label{ssec:framework_AC}

We model the wireless RA network optimization problem as an MARL task defined by the MDP tuple $\{\mathcal{S}, \{\mathcal{A}_k\}_{k\in\mathcal{K}}, P, \{\varrho_k\}_{k\in\mathcal{K}}\}$, where $\mathcal{S}$ is the global state space and $\mathcal{A}_k$ is the action set for device $k$.
The term $P:\mathcal{S} \times \mathcal{A} \times \mathcal{S} \rightarrow [0,1]$ denotes the state transition probability, where $\mathcal{A} = \prod_{k\in\mathcal{K}} \mathcal{A}_k$ is the joint action set of all agents.
Lastly, $\varrho_k:\mathcal{S} \times \mathcal{A} \rightarrow \mathbb{R}$ is the local reward function for device $k$.

Since most of time slots $\{1,\ldots,T\}$ are occupied by the RA protocol steps, such as CCA, packet transmission, and waiting upon ACK, the actual time slots related to MARL procedure are confined.
Hence, we consider our MDP to only progress over time slots, where an action is taken by the devices, which we denote using $n$.

At an MDP step $n$, each device $k$ at a global state $s_n\in\mathcal{S}$ selects an action $a_{k,n}\in\mathcal{A}_k$ according to its local policy $\pi_k$, i.e., $a_{k,n}\sim\pi_k(\cdot\vert s_n)$.
Due to the coupled nature of the devices' actions, the state transitions from $s_{n}$ to $s_{n+1}$ via $P(s_{n+1}|s_{n},\boldsymbol{a}_{n})$, where $\boldsymbol{a}_{n}=[a_{1,n},a_{2,n},\ldots,a_{K,n}]$ is the joint action vector.
Each device $k$ then receives a reward $r_{k,n}=\varrho_k(s_{n},\boldsymbol{a}_{n})$.

The AC framework employs two sets of function approximators: \textit{actors} and \textit{critics}.
Each device $k$ is equipped with an actor network parameterized by weights $\theta_k$, which defines its local policy $\pi_{\theta_k}$.
Let $\theta = [\theta_1^{\top},\theta_2^{\top},\ldots,\theta_K^{\top}]^\top$ denote the joint weight vector of all $K$ actors.
The goal of MARL is to optimize $\theta$ to maximize the expected long-term discounted cumulative reward defined as $J(\theta)= \mathbb{E}_{\pi_{\theta}}\!\left[\sum_{n=0}^{\infty}\gamma^n\bar{r}_{n}\right]$, where $\overline{r}_n=\frac{1}{K}\sum_{k=1}^{K}r_{k,n}$ is the global average reward, and $\gamma\in[0,1]$ is the discount factor.
For each device $k$, the state-value function $V_{\theta} (s) = \mathbb{E}_{\pi_\theta}\left[\sum_{n=0}^{\infty}\gamma^n\overline{r}_n\Big|s_{0}=s\right]$ is estimated using a critic network of parameters $w_k$.
Applying the Bellman equation~\cite{Sutton98}, the estimated function can be expressed as $V_{w_k} (s) = \mathbb{E}_{\pi_\theta}\left[\overline{r} + \gamma V_{w_k}(s')\right]$.
During training, each actor selects an action based on its current policy, and the corresponding critic evaluates the resulting reward.
The actor policy is then updated using the policy gradient $\mathbb{E}_{s,a} [\nabla_{\theta_k}\log\pi_{\theta_k}(a_k|s)\cdot \text{Adv}_{\theta}(s,\boldsymbol{a})]$,
where $\text{Adv}_{\theta}(s,\boldsymbol{a}) = \bar{r} + \gamma V_\theta (s') - V_\theta(s)$ is the advantage function.
This update encourages the actor to select actions that yield higher expected future rewards relative to the baseline value.

\vspace{-2mm}
\subsection{MDP Formulation for MARL-based Wireless RA Network Optimization}\label{ssec:framework_MDP}

We now design the MDP components for our RA network optimization.
First, the state at network time slot $t$ is defined as
\vspace{-1mm}
\begin{equation}
    s_{t}=\left(\{\widehat{\xi}_{k}\}_{k\in\mathcal{K}}, \{\widehat{\ell}_{k,t}\}_{k\in\mathcal{K}}, q_{t}\right),
    \label{eq:MDP_state}
    \vspace{-1mm}
\end{equation}
where $\widehat{\xi}_{k}=\frac{\xi_{k}-\min_k\xi_{k}}{\max_k\xi_{k}-\min_k\xi_{k}}$ is the min-max normalized long-term signal-to-noise ratio (SNR) for device $k$, $\widehat{\ell}_{k,t}=\ell_{k,t}/\lambda $ is the normalized AoI with $\lambda$ as a scaling factor, and $q_{t}\in\{0,1\}$ is the binary channel occupancy indicator.
We use dB scale for SNR, i.e., $\xi_{k}=\mathbb{E}_t\left[10\log_{10}\frac{p_{k} \|\boldsymbol{g}_{k,t}\|^2_2}{BN_0}\right]=10\log_{10}\frac{p_{k} \eta_k}{BN_0}$.
For each device $k$, $\widehat{\xi}_{k}$ and $\widehat{\ell}_{k,t}$ are local information.
We assume that long-term SNRs of other devices $\{\widehat{\xi}_{k'}\}_{k'\in\mathcal{K}\setminus k}$ are available via local information exchanges or broadcast that take place at a much slower rate.
Following the assumptions in~\cite{Yu19,Guo22,He24}, we consider AoIs of other devices $\{\widehat{\ell}_{k',t}\}_{k'\in\mathcal{K}\setminus k}$ to be observable from the broadcast ACK packets.
Note that $q_t$ is also observable during the CCA phase.

Each device operates in a discrete action space $\mathcal{A}_k=\{0, 1\}$, where $0$ indicates the decision to \textit{wait} and $1$ represents the decision to \textit{transmit}.
Hence, at each given time slot $t$, device $k$ relies on its policy $\pi_k$ and selects its action $a_{k,t}\in\{0, 1\}$.

We now define local reward functions corresponding to each downstream task.
First, for device $k$ at time slot $t$, the local reward for fair-AoI task is defined as
\vspace{-1mm}
\begin{equation}
    r_{k,t}^{\mathsf{FA}} = -\frac{\ell_{k,t}}{\omega_\mathsf{FA}K},
    \label{eq:MDP_reward_AoI}
    \vspace{-1mm}
\end{equation}
where $\omega_\mathsf{FA}K$ is a scaling factor to prevent the reward from increasing with the network size.
Since the value of $\ell_{k,t}$ becomes zero only upon successful transmission, this reward encourages devices to avoid collisions.
Moreover, with~\eqref{eq:MDP_reward_AoI}, prioritizing devices with higher AoI values leads to greater marginal reward gains.
Hence, maximizing~\eqref{eq:MDP_reward_AoI} promotes AoI fairness across the network.

Next, the local reward of device $k$ at time slot $t$ for max-rate task is given by
\begin{equation}
    r_{k,t}^{\mathsf{MR}} = \frac{x_{k,t}}{\omega_\mathsf{MR}} + D(x_{k,t}),
    \label{eq:MDP_reward_sum}
\end{equation}
where $\omega_\mathsf{MR}$ is a scaling parameter, and $D(x_{k,t})$ is a penalty function:
\begin{equation}
    \hspace{-1mm}D(x_{k,t}) =
    \begin{cases}
        \log\left(\frac{x_{k,t}+\omega_\mathsf{P1} R_\mathsf{min}}{(1+\omega_\mathsf{P1})R_\mathsf{min}}\right)^{\omega_\mathsf{P2}/K}, & \text{if } x_{k,t} \leq R_\mathsf{min}, \\
        0, & \text{otherwise},
    \end{cases}
    \label{eq:MDP_penalty}
\end{equation}
with scaling factors $\omega_\mathsf{P1}$ and $\omega_\mathsf{P2}$.
Note that~\eqref{eq:MDP_penalty} is bounded in the interval $[\log(\omega_\mathsf{P1}/(1+\omega_\mathsf{P1}),  0]$ and designed to only penalize devices that fail to meet the threshold $R_\mathsf{min}$.
In other words, this reward drives each device to maintain the minimum rate requirement while maximizing the sum-rate.

Lastly, we adopt the $\alpha$-fairness utility function~\cite{Yu19} and define the local reward of device $k$ at time slot $t$ for fair-rate task as
\begin{equation}
    r_{k,t}^{\mathsf{FR}} = \frac{1}{1-\alpha_\mathsf{FR}}\left(\frac{x_{k,t}}{\omega_\mathsf{FR1}/K}+\omega_\mathsf{FR2}\right)^{1-\alpha_\mathsf{FR}},
    \label{eq:MDP_reward_fair}
\end{equation}
where $\alpha_\mathsf{FR}\in(1,\infty)$ controls the fairness, and $\omega_\mathsf{FR1}$ and $\omega_\mathsf{FR2}$ are scaling factors.
Larger values of $\alpha_\mathsf{FR}$ impose stricter fairness.

To reflect the condition of fully decentralized MARL, we strictly define our reward functions to be locally computable.
For the consensus step, we let $r_{k,t} \in \{r_{k,t}^{\mathsf{FA}}, r_{k,t}^{\mathsf{MR}}, r_{k,t}^{\mathsf{FR}}\}$ to be exchanged across the devices through local communication links defined by $\mathcal{G}$.

\begin{figure}[!t]
    \centering
    \includegraphics[width=1\linewidth]{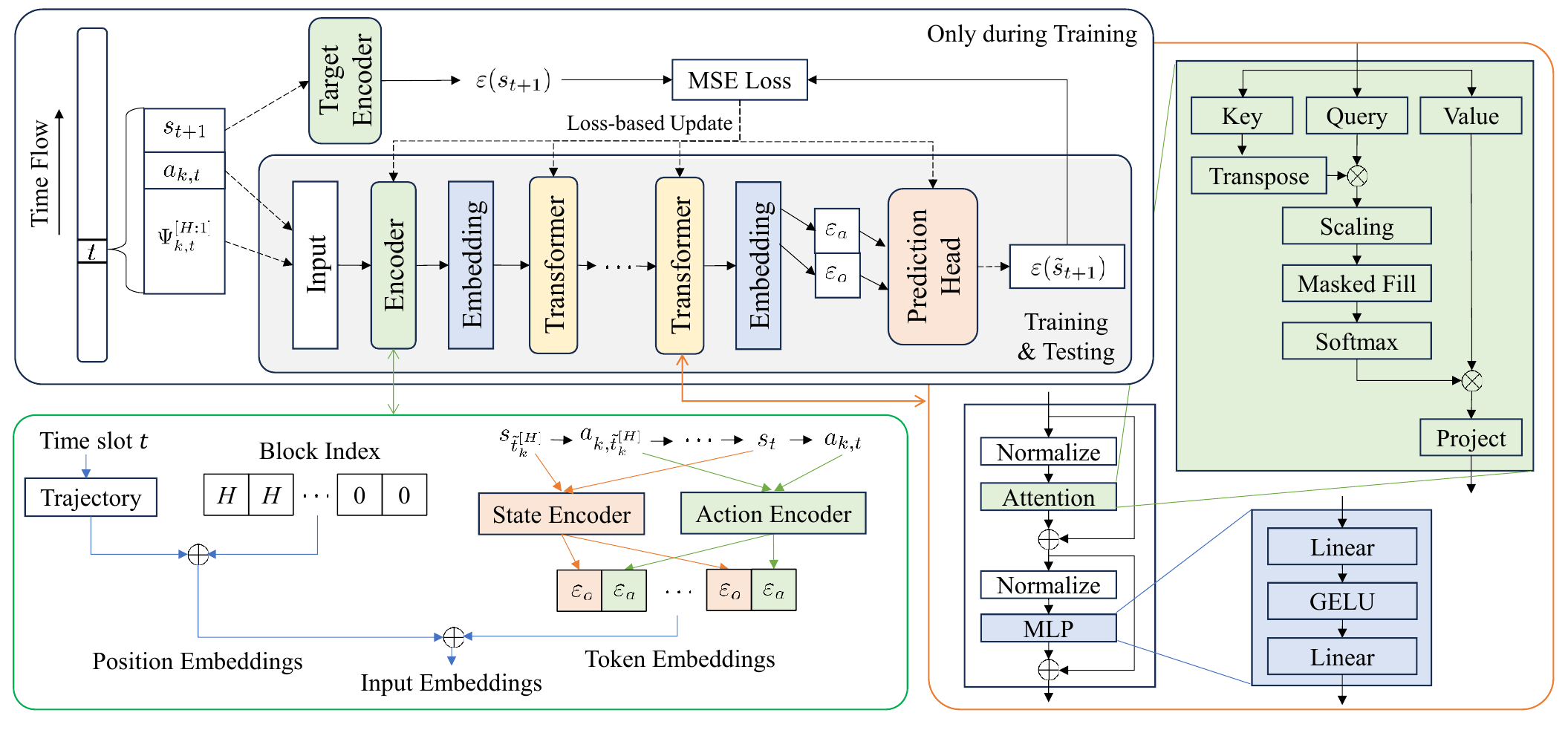}
    \caption{The architecture of FM as a reward-agnostic critic backbone. The model takes a state-action history and extracts embeddings via state and action encoders. Transformer blocks are used to process the embeddings. The output of a prediction head is sent to the reward-specific critic head.}
    \Description{.}
    \label{fig:FM}
    \vspace{-4mm}
\end{figure}

\subsection{Foundation Model Architecture Design}\label{ssec:framework_FM}

We present the detailed design of our FM, which supports the AC framework described in Section~\ref{ssec:framework_AC} as a reward-agnostic critic backbone.
Typical RL models rely on reward signals as supervision, i.e., rewards serve as targets that guide an agent to train and learn desirable behaviors.
In contrast, our FM is pretrained in a {\em self-supervised} manner and thus does {\em not} require reward labels.
This allows our FM to be reward-agnostic and transferable across different RL tasks.

In this work, we leverage the FM to enhance our MARL-based optimization for wireless RA-based networks.
During pretraining, the FM captures useful feature representations that effectively model the underlying environment dynamics, which can improve learning efficiency over multiple downstream tasks (e.g., three wireless RA objectives defined in Section~\ref{ssec:system_objectives}).
Our FM pretraining is inspired by the SMART framework~\cite{Sun23}, which uses a control transformer architecture to facilitate self-supervised learning across different pretraining goals.
Similar to SMART, we employ encoders to extract meaningful embeddings from the inputs and utilize the attention mechanism within transformers to perform a prediction task.
In this paper, we specifically adopt the \emph{forward dynamics prediction} task for self-supervised pretraining, which aligns naturally with the core principle of RL to optimize future outcome predictions.

The architecture of our FM is shown in Figure~\ref{fig:FM}, which consists of three main components: i) encoders, ii) transformers, and iii) a prediction head.
The encoders first map the input into a sequence of token embeddings, which are augmented with position embeddings to convey temporal information.
The position embeddings are designed to capture both the relative order within the history window and the absolute position within the global time trajectory, allowing the model to effectively understand the environment dynamics.
The combined embeddings are then passed through a multi-transformer module, where each block consists of a self-attention layer and a multi-layer perceptron (MLP).
The MLP uses the Gaussian error linear unit (GELU) as its activation function.
The output from the transformer module is fed into the prediction head, which generates an embedding that represents next-step state.

During the pretraining, the FM is optimized such that the {\em loss between the predicted and actual next-state embeddings} is minimized.
By integrating the self-supervised FM, we allow our MARL agents to utilize effective representations of the environment dynamics during their online learning.
Note that the pretraining of FM does not require explicit reward signals, making the model generalizable to different downstream RL tasks.

Similar to~\cite{Yu19,Guo22,He24}, we assume that each device can store the $H$ latest states and actions.
Let $\tilde{t}_{k}^{[h]}$ denote the time slot when the $h$-th latest action was taken by device $k$.
Then, the state-action history at time slot $t$ can be written as
\vspace{-1mm}
\begin{equation}
    \Psi_{k,t}^{[H:1]} = \left\{s_{\tilde{t}_{k}^{[H]}},a_{k,\tilde{t}_{k}^{[H]}},\ldots,s_{\tilde{t}_{k}^{[1]}},a_{k,\tilde{t}_{k}^{[1]}},s_{t}\right\}.
    \label{eq:MDP_hist}
    \vspace{-1mm}
\end{equation}
We integrate~\eqref{eq:MDP_hist} as an input to both actor and critic.
This enables the model to better reflect the temporal dynamics of the wireless RA environment and improve its learning stability and performance.

\vspace{-2mm}
\subsection{Foundation Model-Aided Decentralized MARL Algorithm Design}\label{ssec:framework_MARL}

We now describe the proposed FM-aided decentralized MARL algorithm for RA network optimization.
Figure~\ref{fig:MARL} illustrates the overall architecture of our decentralized AC framework.
Each device $k\in\mathcal{K}$ updates its local actor and critic networks using its own experience.
To capture temporal dependencies, each device maintains a history buffer that stores past observation–action pairs and uses this history as input during learning.
As described in Section~\ref{sec:system}, neighboring devices can exchange limited information.
In our design, each device shares only {\em scalar local rewards} with its neighbors, which incurs significantly lower communication costs compared to exchanging full critic parameters, as in~\cite{Zhang18,Chen22,Hairi22}.
Specifically, the critic network consists of two parts: \textit{a reward-agnostic FM} and a reward-specific head.
The FM is pretrained using observation–action sequences and remains fixed during the RL phase.
This design eliminates the need for online updates of the high-capacity backbone, reducing computational and communication overhead while leveraging rich representations of system dynamics.

The overall procedure of the proposed algorithm for a $K$-device RA-based MAC layer is summarized in Algorithm~\ref{alg:main}.
Each device operates under the LBT protocol and continuously monitors the channel state.
Meanwhile, each device updates its local status based on transmission outcomes (e.g., successful transmission or collision)~\textbf{(Lines 6 - 10)}.
When the channel is sensed to be idle, the device observes the current state $s_{t}$ and makes a transmission decision $a_{k,t}$ using its local policy conditioned on the past action-observation history and current observation (i.e., $\Psi_{k,t}^{[H:1]}$).
After executing the action, the tuple $\{s_t, a_{k,t}, r_{k,t}\}$ is stored in the history buffer \textbf{(Lines 11 - 17)}.
This interaction process is repeated over $T$ time slots.

\begin{figure}[!t]
    \centering
    \includegraphics[width=1\linewidth]{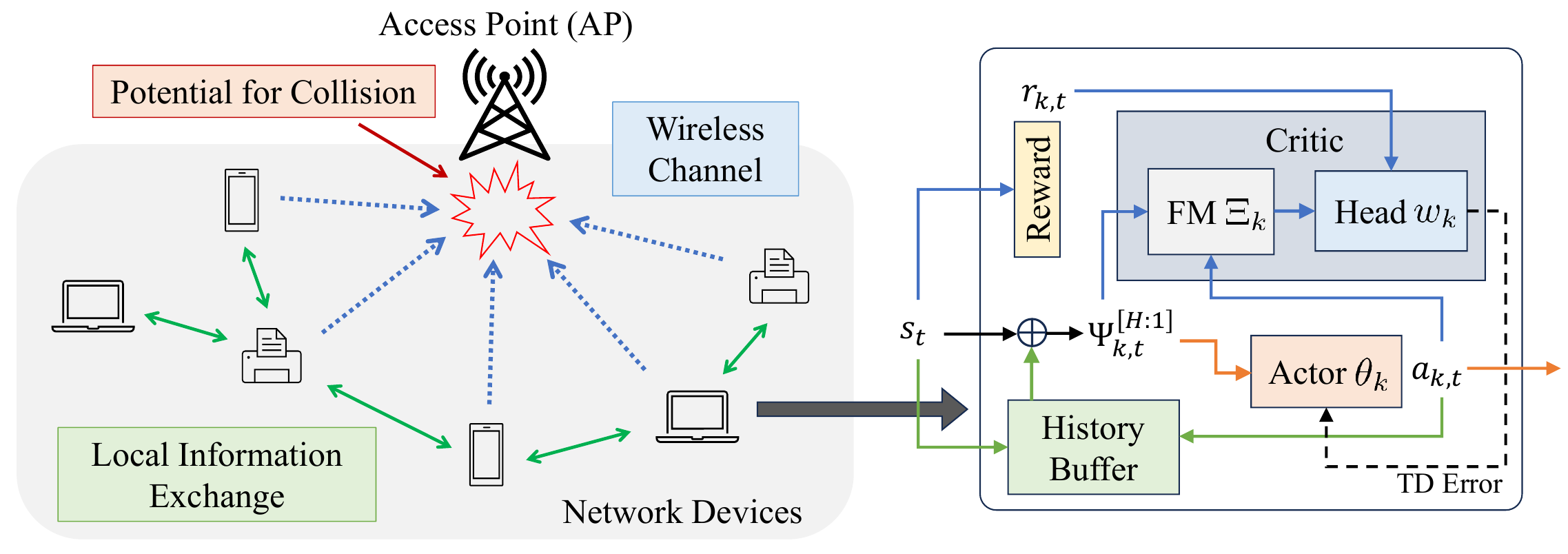}
    \caption{A visual representation of our fully decentralized MARL framework. Both the actor and critic are trained in a decentralized manner. The critic contains two distinct parts: a reward-agnostic backbone for forward dynamics prediction and a reward-specific head for task-specific finetuning.}
    \Description{.}
    \label{fig:MARL}
    \vspace{-6mm}
\end{figure}

\begin{algorithm}[!ht]
	\caption{FM-aided Decentralized MARL for RA Network Optimization.}
    \label{alg:main}
    \textbf{Input:} device set $\mathcal{K}$, pretrained FM weights $\{\Xi_k\}_{k\in\mathcal{K}}$, time horizon length $T$, history length $H$, consensus weight matrix $\mathbf{C}$, consensus iteration count $G$, number of samples to collect $N_B$, actor rate $\alpha$, critic rate $\beta$  \\
    \textbf{Initialize:} For all $k\in\mathcal{K}$, actor weights $\theta_k$, critic weights $w_k$, ready-for-update status $u_k=\text{False}$, sample count $b_k=0$ \\
    \textbf{Load:} FM weights $\Xi_k$ for all $k\in\mathcal{K}$ \\
    \For{$t=1, 2 \dots, T$}{
        \For{$k \in \mathcal{K}$}{
            Observe $q_t$ via clear channel assessment (CCA) \\
            \If{ACK received}{
                $\ell_{k,t} \leftarrow 0$
            }
            \Else{
                $\ell_{k,t} \leftarrow \ell_{k,t} + 1$
            }
            \If{$q_t=0$}{
                Acquire state $s_{t}$ and reward $r_{k,t}$\\
                Select $a_{k,t}\sim\pi_{\theta_k}(\:\cdot\:|\Psi_{k,t}^{[H:1]})$ \\
                \If{$a_{k,t}=1$}{
                    Initiate a packet transmission
                }
                Store $s_{t}$, $a_{k,t}$, and $r_{k,t}$ in the history buffer \\
                $b_k \leftarrow b_k + 1$ \\
            }
            \If{$b_k = N_b$}{
                $u_k \leftarrow \text{True}$
            }
        }
        \If{$u_k=\text{True}, \forall k\in\mathcal{K}$}{
            $\overline{\mathbf{R}}^{(t)}=\mathbf{C}^G\mathbf{R}^{(t)}$ \\
            \For{$k \in \mathcal{K}$}{
                $\bar{w}_{k,0} \leftarrow w_k$ \\
                \For{$b = 1,2,\ldots,N_B$}{
                        $\delta_{k,b} \leftarrow \overline{\mathbf{R}}^{(t)}_{[k,b]} + \gamma V_{\bar{w}_{k,b-1}}(\phi_{\Xi_k}(\Psi_{k,\tilde{t}_{k}^{[b-1]}}^{[H+b-1:b]}))- V_{\bar{w}_{k,b-1}}(\phi_{\Xi_k}(\Psi_{k,\tilde{t}_{k}^{[b]}}^{[H+b:b+1]}))$ \\
                        $\tilde{w}_{k,b} \hspace{-.5mm} \leftarrow \bar{w}_{k,b-1} - \beta\delta_{k,b} \cdot \nabla  V_{\bar{w}_{k,b-1}}(\phi_{\Xi_k}(\Psi_{k,\tilde{t}_{k}^{[b]}}^{[H+b:b+1]}))$ \\
                        $\bar{w}_{k,b} \leftarrow \argmin_{w\in\mathcal{B}(B)}\|w-\tilde{w}_{k,b}\|_2$ \\
                        $w_k \leftarrow \frac{b}{b+1}w_k + \frac{1}{b+1}\bar{w}_{k,b}$
                    }
                    $\theta_k \hspace{-.5mm} \leftarrow \hspace{-.5mm} \theta_k \hspace{-.5mm} + \hspace{-.5mm} \frac{\alpha}{N_B} \hspace{-1mm} \sum_{b=1}^{N_B}\delta_{k,b} \hspace{-.5mm} \cdot \hspace{-.5mm} \nabla\hspace{-.5mm}\log\pi_{\theta_k}(a_{k,\tilde{t}_{k}^{[b]}}\vert\Psi_{k,\tilde{t}_{k}^{[b]}}^{[H+b:b+1]}) \hspace{-5mm}$ \\
                $u_k \leftarrow \text{False}$ \\
            }
        }
    }
    \textbf{Output:} $\theta_k$ for all $k\in\mathcal{K}$ \\
\end{algorithm}

Each time the devices are ready to update their model weights, the consensus \textbf{(Line 21)} first takes place for reward exchange.
Each consensus step consists of $G$ rounds of communication, where weight-based averaging is performed in each round.
Then, each device updates its actor and critic parameters after computing the temporal-difference (TD) error \textbf{(Lines 23 - 29)}.
While the actor takes the input directly from the observation-action history, the critic processes it via the FM backbone to obtain a latent representation $\phi_{\Xi_k}(\Psi_{k,t})$, which is then fed into the reward-specific head for value estimation.
Extending our prior work~\cite{Oh2025} to FM-aided MARL, we adopt a batch-based update scheme to facilitate the theoretical analysis in Section~\ref{ssec:framework_analysis}.
In particular, the critic update incorporates projection onto a bounded parameter set and an averaging step \textbf{(Lines 27 - 28)} to ensure that each update remains within a controlled region and allow convergence guarantees~\cite{Xu20,Zhang25}.

\vspace{-2.5mm}
\subsection{Theoretical Performance Analysis}\label{ssec:framework_analysis}
\vspace{-.5mm}

Our prior work on RA network optimization via decentralized MARL~\cite{Oh2025} establishes convergence guarantees under linear value function approximation, which has been the dominant analytical framework for studying MARL convergence~\cite{Chen22,zhang2021taming,Hairi22}.
However, linear approximations are insufficient to capture the expressive behavior of deep neural networks (DNNs), which are widely used in modern RL applications.
To address this gap, we build upon our previous analytical framework in~\cite{Zhang25} and extend our results to the FM-aided setting where value functions are {\bf approximated using DNNs}.
A key distinction of this work from~\cite{Zhang25} is that their analysis assumes consensus over the entire critic model, whereas our approach requires {\em only local reward exchange} among neighboring agents.
By restricting consensus to scalar rewards rather than high-dimensional model parameters, our framework significantly improves practical deployability and reduces communication overhead.
Also, we address the impact of {\bf integrating pretrained FMs} in the actor-critic pipeline, where we perceive the power of self-supervised learning as the prediction capability on next states.

We now formally describe the DNN architecture used in our analysis.
For each device $k$, the critic maintains a DNN-based approximation $V_{w_k}$ of the true value function $V_{\theta}$.
We assume the network has width $M_w$ and depth $M_d$.
Let $x_\text{in}\in\mathbb{R}^d$ denote the input features vector.
The forward pass is defined as:
\vspace{-1mm}
\begin{align}
    x^{(0)} &= W^{(0)}x_\text{in}, \\
    x^{(h)} &= \frac{1}{\sqrt{M_w}}\mathsf{ReLU}(W^{(h)}x^{(h-1)})\text{ for } h=1,2,\ldots,M_d \\
    y &= b^{\top}x^{(M_d)},
\end{align}

\vspace{-2mm}
\noindent where $W^{(0)}\in\mathbb{R}^{m\times d}$, $W^{(h)}\in\mathbb{R}^{m\times m}$, and $b\in\mathbb{R}^m$ are the network parameters, and $\mathsf{ReLU}(x) := \max\{0,x\}$ is the ReLU activation function.
The actors share the same architecture; however, a softmax layer is appended to yield a probability distribution over actions.

With $n$ denoting the MDP iteration index, we establish the theoretical convergence guarantee of Algorithm~\ref{alg:main}.
First, we provide the required assumptions in the following.
\vspace{-1mm}
\begin{assumption}[Bounded Reward] \label{ass:r-max}
    For each downstream task, there exists a constant $r_{\text{max}}>0$ such that $|r_{k,n}| \le r_{\text{max}}, \forall n\ge 0, k\in\mathcal{K}$.
\end{assumption}
\vspace{-3mm}
\begin{assumption}[Mixing Time] \label{ass:mixing}
     There exists a stationary distribution $\zeta$ for $(s,a)$, and positive constants $\kappa$ and $\rho\in(0,1)$, such that $\sup_{s\in\mathcal{S}}\|P(s_n,a_n|s_0=s) - \zeta_\theta\|_{TV}\leq \kappa\rho^n, \forall n\ge 0$.
\end{assumption}
\vspace{-3mm}
\begin{assumption}[Lipschitz Continuity] \label{ass:smoothness}
    There exist some positive constants $L_J$, $L_W$, $L_V$, and $M_g$ such that, for any $\theta$ and $\theta'$, any $w$ and $w'$, any $x$ and $x'$, and any $s$ and $a$, we have $|J(\theta)-J(\theta')| \leq L_J\|\theta-\theta'\|_2$, $|V_{w}(x) - V_{w'}(x)| \leq L_W\|w-w'\|_2$, $|V_{w}(x) - V_{w}(x')\| \leq L_V\|x-x'\|_2$, and $\|\nabla_{\theta}\log\pi_{\theta}(a|s)\|_2 \leq M_g$.
\end{assumption}
\vspace{-3mm}
\begin{assumption}[Consensus Matrix] \label{ass:consensus}
    The consensus weight matrix $\mathbf{C}$ is doubly stochastic. Additionally, for all $k,k'\in \mathcal{K}$, there exists a positive constant $\nu > 0$ such that (i) $c_{kk} \geq \nu$ and (ii) $c_{kk'} \geq \nu$ whenever devices $k$ and $k'$ are connected.
\end{assumption}
\vspace{-1mm}
Assumption 4 ensures that information is exchanged with uniformly positive weights among connected agents, facilitating convergence of the local consensus estimates toward the globally averaged information.
\begin{assumption} \label{ass:epsilon-bias}
    For \,any policy \,$\theta$, there exist positive constants \, $\sigma$ \, and \, $\epsilon_\text{bias}$ \, such \, that \, $\mathbb{E}\big[\mathbf{Adv}_{\theta}(\Psi,a) - (1-\gamma)$$(\mathbf{F}(\theta)^{-1}$$\nabla_{\theta} J(\theta))^{\top}$ 
    
    \noindent $\nabla_{\theta}\log\pi_{\theta}(a|\Psi) \big]^2$ $\leq \epsilon_\text{bias}$,
    where $(\Psi,a)$ follows $\nu(\pi^*)$, which is a stationary distribution under the optimal policy $\pi^*$, and $\mathbf{F}(\theta) = \mathbb{E}_{\nu(\theta)}[\nabla_{\theta}\log\pi_{\theta}(a|\Psi)$ $\nabla_{\theta}\log\pi_{\theta}(a|\Psi)^\top] + \sigma \mathbf{I}$.
\end{assumption}
\vspace{-3mm}
\begin{assumption}[Universal Approximation] \label{ass:epsilon-critic}
    Suppose for any policy $\theta$, the critic estimation $V(\cdot;w^*)$, where $w^*$ is a stationary point, can be close enough to the ground truth $V_{\theta}(\cdot)$, i.e., there exists a positive constant $\epsilon_{\text{critic}}$ such that $\max_{\theta,x}|V(x;w^*) - V_{\theta}(x)| \leq \epsilon_{\text{critic}}$.
\end{assumption}
We state the main theoretical result as below.
\begin{theorem}
    With a pretrained FM of estimation error $\epsilon^\text{FM}_\text{est}$ and approximation error $\epsilon^\text{FM}_\text{approx}$ fixed and integrated as a reward-agnostic critic backbone, by choosing step-size parameters $\alpha=\frac{7}{2\sqrt{\mu}}$, $\alpha_n = \frac{\alpha}{n}$, $\beta = 1/\sqrt{N_B}$, and network/algorithm parameters satisfying $m = \Omega(d^{3}M_d^{-\frac{11}{2}})$, $B = \Theta(M_w^{\frac{1}{32}}M_d^{-6})$, and $N_B = \Omega(M_d^4)$, it holds with probability at least $1- \exp(-\Omega(\log^2M_w))$ that Algorithm~\ref{alg:main} achieves:
    \begin{align}
        &\mathbb{E}\left[J(\theta^*) - J(\theta_{N_T})\right] = \underbrace{\mathcal{O}\big(N_T^{-1}\big)}_{(a)} + \underbrace{\mathcal{O}(K^{\frac{3}{2}} (1-\nu^{K-1})^{G})}_{(b)} \nonumber \\
        & + \underbrace{\mathcal{O}\big(K^{\frac{1}{2}}(BM_d^{\frac{5}{2}}N_B^{-\frac{1}{4}} + B^{\frac{4}{3}}M_w^{-\frac{1}{12}}M_d^4)\log^{\frac{3}{2}}M_w\log^{\frac{1}{2}} N_B\big)}_{(c)} + \underbrace{\mathcal{O}\big(\sqrt{\epsilon_\text{bias}}\big)}_{(d)} \nonumber \\
        & + \underbrace{\mathcal{O}(K^{\frac{1}{2}}\epsilon^\text{FM}_\text{est})}_{(e)} + \underbrace{\mathcal{O}(K^{\frac{1}{2}}\epsilon^\text{FM}_\text{approx})}_{(f)} + \underbrace{\mathcal{O}\big(K^{1/2}\epsilon_{\text{critic}}/N_B^{1/2}\big)}_{(g)}. \label{eq:theorem}
    \end{align}\vspace{-1em}
    \label{thm:main}
\end{theorem}
\textit{Proof Sketch}. To bound the suboptimality gap $J(\theta^*) - J(\theta_{N_T})$, we focus on analyzing the discrepancy between the update direction $d_n$ and the true policy gradient $\nabla J(\theta_n)$ at each iteration.
We decompose this discrepancy into four components: (i) value function approximation error, (ii) actor update error induced by imperfect reward consensus, (iii) the mismatch between TD error and the advantage function, and (iv) intrinsic representation error of the function approximators.
Each component is bounded individually to obtain the terms in~\eqref{eq:theorem}.
In particular, the value function approximation error is further decomposed into contributions from imperfect critic updates, consensus inaccuracies, and limitations of the pretrained FM representations.
These terms are controlled using the stated assumptions and standard properties of overparameterized DNNs.
Due to space limitation, we relegate the full proof to our online technical report \cite{proof}
\qed

Theorem~\ref{thm:main} provides several key insights.
{\bf First}, term~$(b)$ shows that imperfect reward consensus induces an optimality gap of the same order as approaches that exchange full critic parameters and TD errors~\cite{Zhang25}.
This indicates that restricting communication to scalar reward exchange preserves performance while significantly improving practicality and reducing communication overhead.
{\bf Second}, term~$(c)$ highlights the role of network architecture.
The convergence bound exhibits stronger sensitivity to the network depth than to its width, suggesting that increasing depth is more effective for improving performance, whereas overparameterization in width yields diminishing returns.
{\bf Third}, terms~$(e)-(g)$ characterize the impact of critic approximation errors, which scale with the number of agents $K$.
In particular, terms~$(e)$ and~$(f)$ correspond to estimation and approximation errors arising from the self-supervised pretraining of the FM, where the learning objective is forward dynamics prediction.
These error components are unique to our framework and reflect the quality of the pretrained, reward-agnostic representations used by the critic.

Theorem~\ref{thm:main} shows that Algorithm~\ref{alg:main} converges to a neighborhood of the global optimum at a rate of $\mathcal{O}(1/N_T)$.
This indicates that operating under nonlinear function approximation and decentralized reward consensus matches the convergence rate of state-of-the-art actor–critic methods with linear function approximation and full critic exchange~\cite{Hairi22}.
Note that our result generalizes~\cite{Oh2025}, which only considered linear approximation, to a more realistic deep learning setting.
Moreover, unlike prior work that guarantees convergence only to stationary points (e.g., in terms of $\|\nabla J(\theta)\|^2$)~\cite{Hairi22,Oh2025}, our analysis characterizes convergence toward {\bf global optimality}.

\begin{figure*}[!t]
    \centering
    \begin{subfigure}{0.32\linewidth}
        \includegraphics[width=1\linewidth]{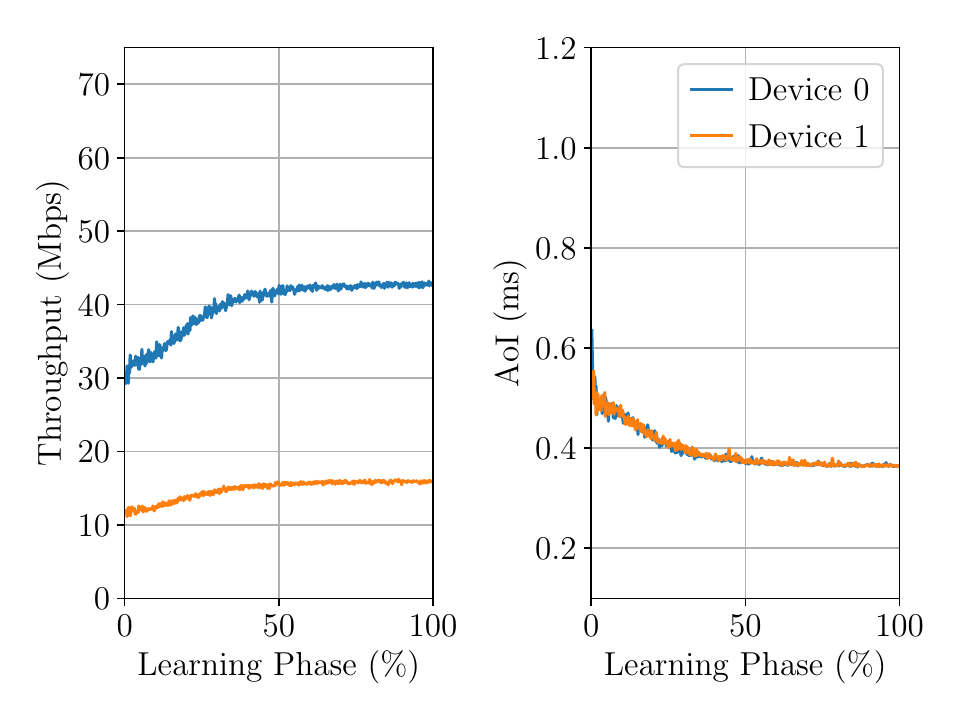}
        \vspace{-.25in}
        \subcaption{Fair-AoI task.}
        \label{fig:fair_AoI}
    \end{subfigure}
    \begin{subfigure}{0.32\linewidth}
        \includegraphics[width=1\linewidth]{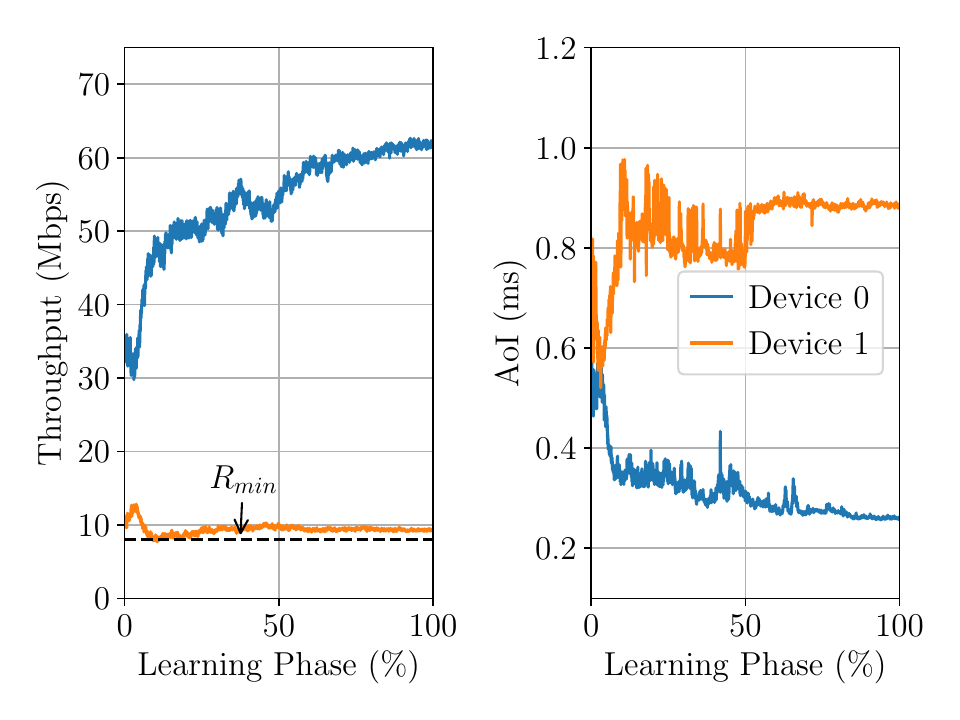}
        \vspace{-.25in}
        \subcaption{Max-rate task.}
        \label{fig:max_rate}
    \end{subfigure}
    \begin{subfigure}{0.32\linewidth}
        \includegraphics[width=1\linewidth]{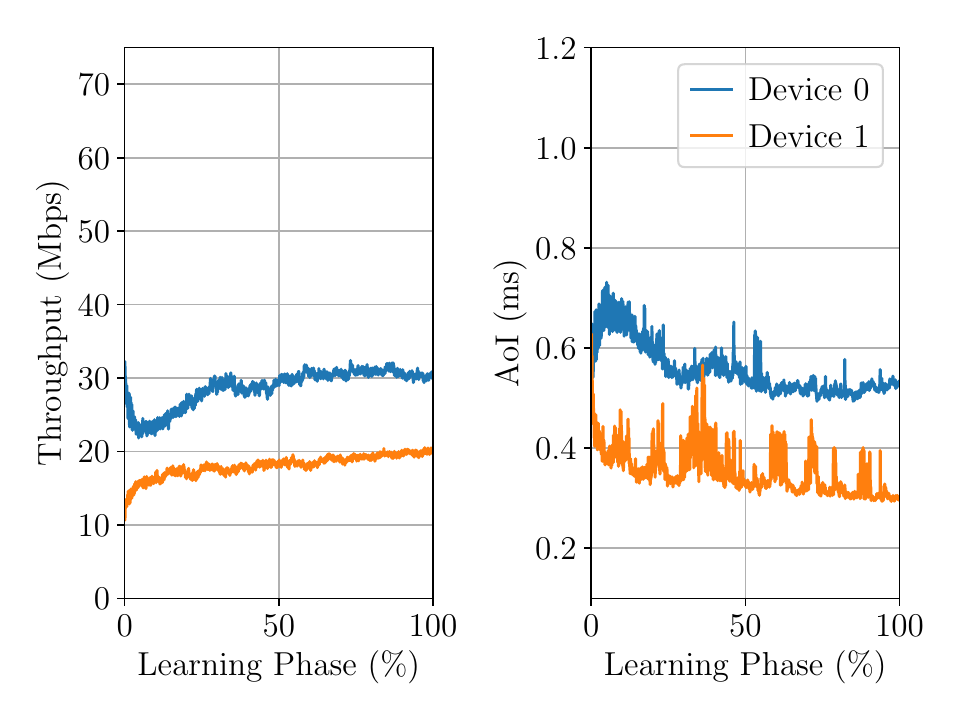}
        \vspace{-.25in}
        \subcaption{Fair-rate task.}
        \label{fig:fair_rate}
    \end{subfigure}
    \vspace{-3mm}
    \caption{Throughput and AoI under end-to-end learning with $K=2$ devices across different RA optimization tasks.}
    \label{fig:downstream_tasks}
    \vspace{-4mm}
\end{figure*}

\vspace{-3mm}
\section{Numerical Evaluation}\label{sec:numerical}

\subsection{Experimental Settings}\label{ssec:numerical_setting}

We consider operating a wireless RA-based network over $T=600$ time slots.
As described in Section~\ref{sec:system}, all devices operate under the LBT mechanism.
Following the IEEE 802.11 protocol, the SIFS and DIFS durations are set to $2$ and $4$ time slots, respectively, with each slot lasting $T_\mathsf{s} = 9$ $\mu s$.
Each data packet transmission requires $\tau = 10$ time slots, while ACK signals occupy $4$ time slots.

We set the number of AP antennas to $N_\mathsf{r}=2$, the bandwidth to $B=10$ MHz, and the transmit power to $p_k=0$ dBm, with the noise spectral density at $-174$ dBm/Hz.
For a given number of devices $K$, device-to-AP distances are uniformly distributed between $2$ m and $14$ m.
For example, with $K=4$, the device distances are $\{d_{k}\}_{k=1}^{K=4}=\{2, 6, 10, 14\}$.
The large-scale fading factor $\eta_{k}$ is modeled as $\eta_{k}=L_\mathsf{o}\left(\frac{d_{k}}{d_\mathsf{o}}\right)^{z}$, where $L_\mathsf{o}=-40$ dB is the reference pathloss, $d_\mathsf{o}=1$ m is the reference distance, and $z=4$ is the pathloss exponent.

The network graph $\mathcal{G}$ is generated using the Watts-Strogatz model~\cite{Watts98}, where each device connects to one neighboring device with zero rewiring probability.
To generate the consensus weight matrix $\mathbf{C}$, we assign equal weights are assigned to each device’s established links, i.e., for each device $k\in\mathcal{K}$, $c_{kk'}=\frac{1}{\vert\mathcal{N}_k\vert+1},\forall k'\in\mathcal{N}_k$, where $\mathcal{N}_k$ denotes the neighbor set of device $k$.
Each consensus step consists of $G=3$ communication rounds.

For the critic-backbone FM, we set the embedding size to $N_\mathsf{e}=64$ such that the encoders, transformer blocks, and prediction head all generate embeddings of dimension $64$.
Both the state and action encoders are implemented as two-layer MLPs with width $64$.
The FM consists of two transformer blocks, each containing a self-attention layer with 8 heads and an MLP of width $4N_\mathsf{e}$ and depth $1$.
The prediction head is a linear transformation of size $2N_\mathsf{e}\times N_\mathsf{e}$ with bias.
No dropout is applied in the FM design.

Each device’s FM is pretrained using more than $100,000$ reward-independent RA-based wireless network samples, each consisting of a tuple $\{t,s_t,a_{k,t}\}$.
We employ stochastic gradient descent (SGD) with a rate of $0.005$, applying $500$ updates with a batch size $256$.
The target encoder is updated using a soft update with a rate of $0.005$.

For both actor and reward-specific critic head, we adopt an MLP block of width $M_w=128$ and depth $M_d=3$.
The actor and critic weights are updated using the SGD optimizer, with learning rates $\alpha=0.001$ and $\beta=0.002$.
Discount factor is set to $\gamma=0.9$.
For the end-to-end training baseline, we employ an MLP block of width $M_w=128$ and depth $M_d=7$ for both actor and critic to ensure a comparable number of parameters for fair evaluation.
Since the state dimension grows proportionally with $K$, we adjust the history length $H$ for different values $K$ to stabilize the input dimension $d$ of the actor–critic network against the FM embedding size $N_e$.

For the MDP parameters, we set $\lambda=60$, $\omega_\mathsf{FA}=15$, $\omega_\mathsf{MR}=24\times10^6$, $\omega_\mathsf{P1}=0.5$, $\omega_\mathsf{P2}=8$, $\omega_\mathsf{FR1}=256\times10^6$, $\omega_\mathsf{FR2}=0.7$, and $\alpha_\mathsf{FR}=12$ for parameter scaling.
These values are heuristically selected to bound the reward to a reasonable range.
All results are averaged over at least $20$ independent runs.

As performance metrics on RA-based wireless networks, we evaluate throughput (in Mbps) and AoI (in ms).
In addition, we assess fairness across devices using the normalized gap (N-Gap) defined as $\text{N-Gap of }\{x\}=(\max\{x\}-\min\{x\})/{\max\{x\}}$.

\begin{figure}[!t]
    \centering
    \begin{subfigure}{0.49\linewidth}
        \includegraphics[width=1\linewidth]{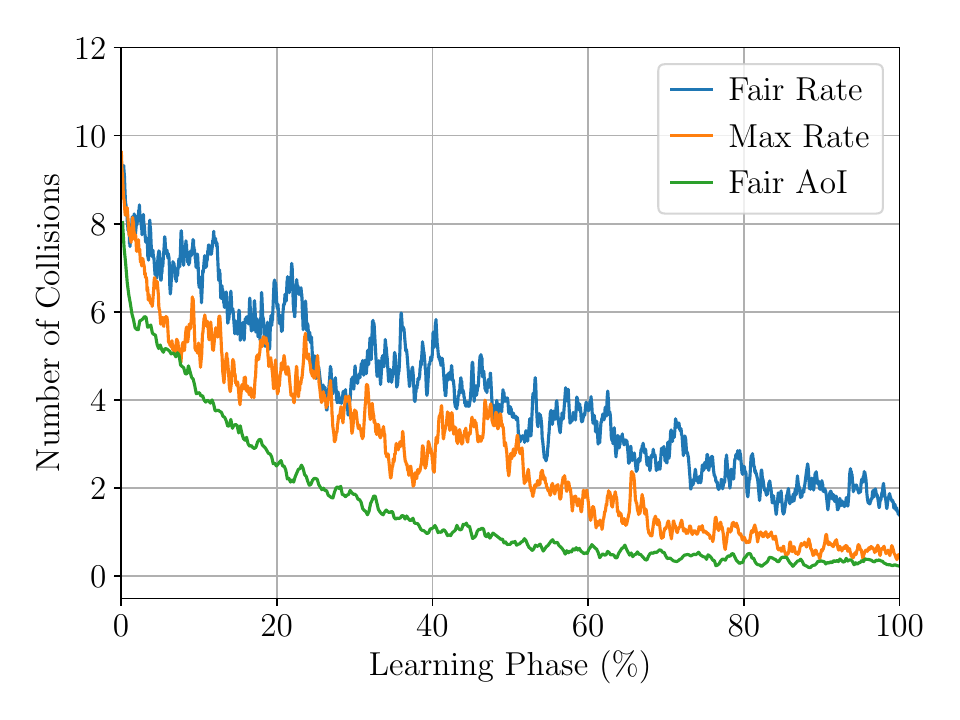}
        \subcaption{Number of Collisions.}
        \label{fig:2_collision}
    \end{subfigure}
    \begin{subfigure}{0.49\linewidth}
        \includegraphics[width=1\linewidth]{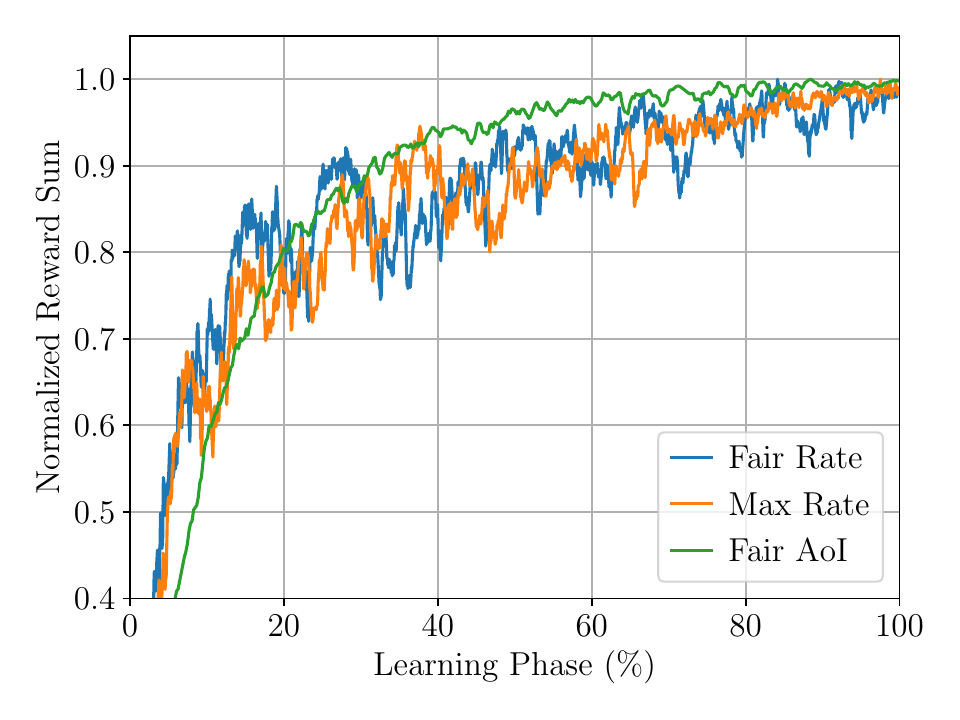}
        \subcaption{Normalized reward sum.}
        \label{fig:2_reward}
    \end{subfigure}
    \vspace{-3mm}
    \caption{Collision and reward performance under end-to-end learning with $K=2$ devices across RA optimization tasks.}
    \Description{.}
    \label{fig:2_others}
    \vspace{-6mm}
\end{figure}

\vspace{-2mm}
\subsection{Experimental Results and Discussion}\label{ssec:numerical_analysis}

\textbf{1) Performance Analysis of the Proposed Decentralized MARL Algorithm on Wireless RA Network Optimization Tasks:}
We first validate Algorithm~\ref{alg:main} and the MDP formulation introduced in Section~\ref{ssec:framework_MDP} by evaluating end-to-end learning performance~\cite{Oh2025}.
In Figure~\ref{fig:downstream_tasks}, we present the throughput and AoI achieved during the learning phase for three downstream tasks.
For ease of illustration, we consider $K=2$ devices. 
The results demonstrate strong learning performance across all tasks.
As expected, the fair-AoI task (Figure~\ref{fig:fair_AoI}) successfully balances AoI across devices.
The max-rate task (Figure~\ref{fig:max_rate}) effectively maximizes the throughput of one device while maintaining the other near the minimum rate constraint $R_\mathsf{min}$.
The fair-rate task (Figure~\ref{fig:fair_rate}) is able to jointly improve throughput of both devices while preserving the fairness.
These observations confirm that the proposed MDP formulation is capable of supporting diverse RA network optimization objectives.

As additional performance metrics, Figures~\ref{fig:2_collision} and~\ref{fig:2_reward} show the number of collisions and the normalized reward sum, respectively.
For improved readability, all results are smoothed using a moving average over $5$ samples.
The results in Figure~\ref{fig:2_collision} show that the fair-AoI task achieves the fastest reduction in collisions, indicating that balancing AoI is the least complex among the three objectives.
This is intuitive, as the fair-AoI objective does not require consideration of channel conditions, which directly influence device data rates.
By contrast, the fair-rate task exhibits the highest collision counts even after many episodes.
This is because its reward function imposes stronger penalties for rate unfairness across devices.
Consequently, the policy that reduces collisions emerges later in training, consistent with Figure~\ref{fig:fair_rate} where device rates increase gradually while maintaining fairness gaps.
In Figure~\ref{fig:2_reward}, we observe a steady increase in the normalized reward sum over the learning process.
This trend validates the effectiveness of our reward functions and confirms that the algorithm successfully optimizes the objectives.

\begin{figure*}[!ht]
    \centering
    \begin{subfigure}{0.32\linewidth}
        \includegraphics[width=1\linewidth]{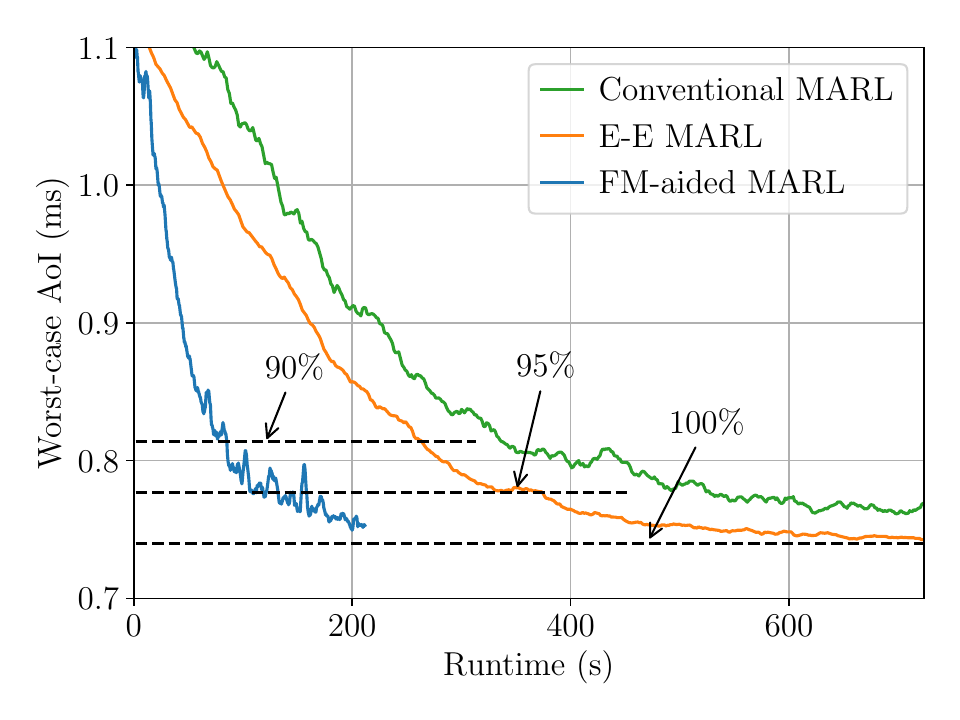}
        \subcaption{Fair-AoI task.}
        \label{fig:4_fair_AoI}
    \end{subfigure}
    \begin{subfigure}{0.32\linewidth}
        \includegraphics[width=1\linewidth]{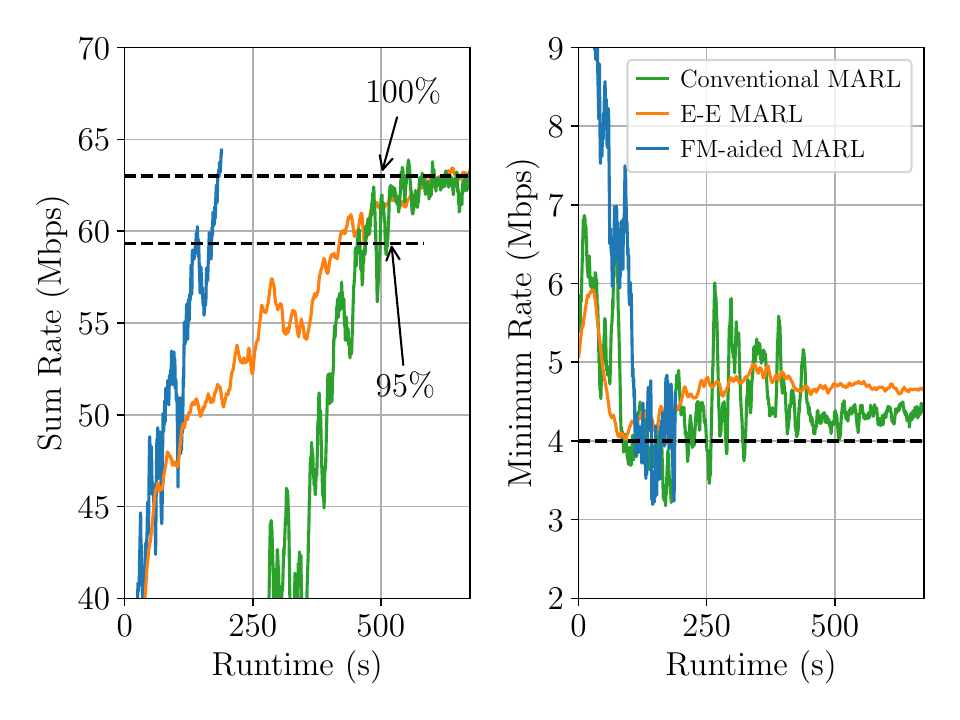}
        \subcaption{Max-rate task.}
        \label{fig:4_max_rate}
    \end{subfigure}
    \begin{subfigure}{0.32\linewidth}
        \includegraphics[width=1\linewidth]{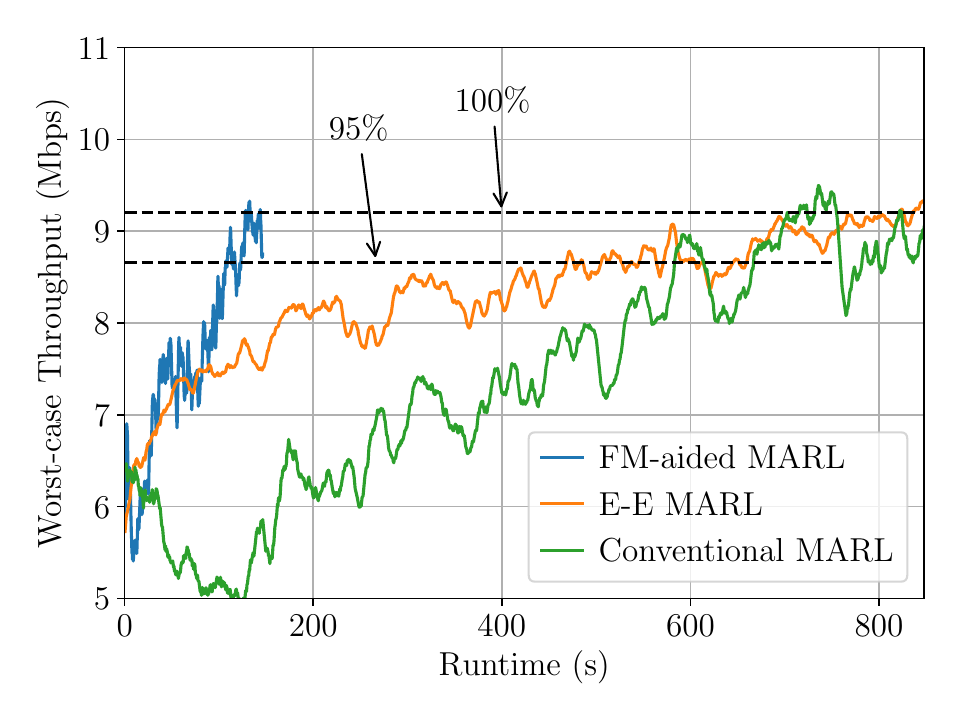}
        \subcaption{Fair-rate task.}
        \label{fig:4_fair_rate}
    \end{subfigure}
    \vspace{-3mm}
    \caption{Performance comparison of conventional, end-to-end (E-E), and FM-aided MARL across different RA optimization tasks with $K=4$. The FM-aided MARL achieves faster convergence than the end-to-end baseline.}
    \Description{.}
    \label{fig:4_comparison}
    \vspace{-2mm}
\end{figure*}

\begin{figure}[!t]
    \centering
    \includegraphics[width=1\linewidth]{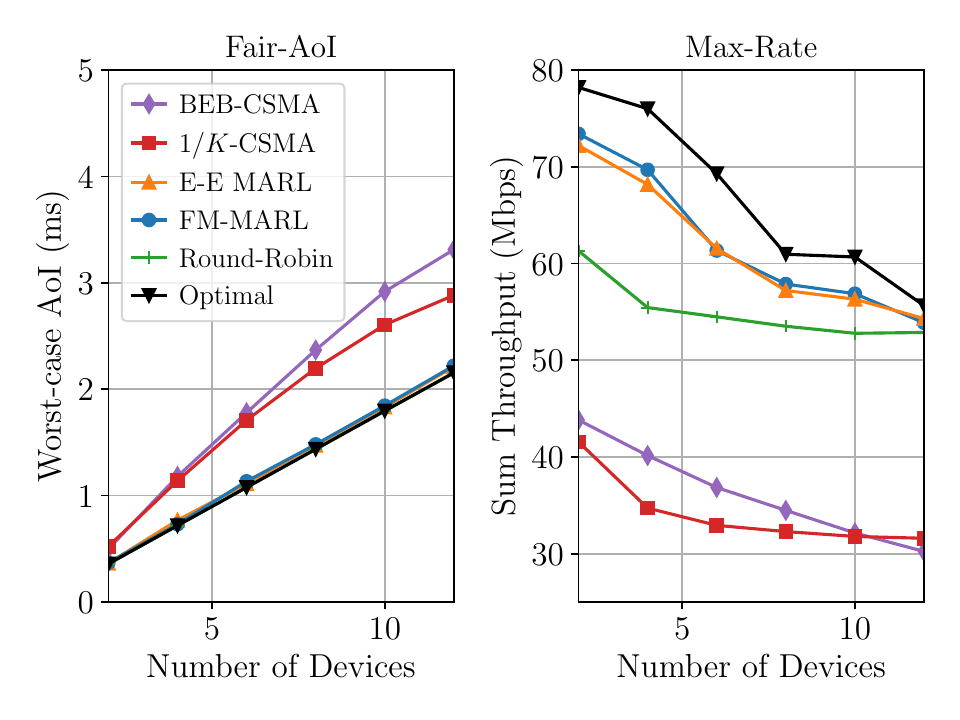}
    \vspace{-10mm}
    \caption{Fair-AoI and Max-rate performance comparison with non-MARL baselines across the number of devices $K$.}
    \Description{.}
    \label{fig:device_scale}
    \vspace{-3mm}
\end{figure}

\textbf{2) Performance Comparison of FM-Aided MARL and Baselines:}
With the end-to-end (E-E) MARL results as a proof of concepts, we are now ready to assess the effectiveness of our FM-aided MARL algorithm by comparing it with E-E MARL and conventional decentralized MARL (conventional MARL), which performs full critic exchanges~\cite{Zhang25}.
Figure~\ref{fig:4_comparison} presents the AoI and throughput over runtime for each downstream task with $K=4$.
For all tasks, FM-aided learning shows superior learning speed against both baselines, approximately improving the convergence speed by at least 55\%.
Note that the performance gap with E-E MARL is attributed to the integration of the FM as a reward-agnostic critic backbone.

In Figure~\ref{fig:device_scale}, we compare our algorithm with several baselines, which include $1/K$-persistent CSMA ($1/K$-CSMA) and CSMA/CA with binary exponential backoff (BEB-CSMA), under the Fair-AoI and Max-rate objectives for different numbers of devices.
We observe that both E-E MARL and FM-MARL consistently outperform the CSMA-based baselines, which rely on stochastic channel-access mechanisms.
Round-robin, where devices are scheduled in turn through centralized control, achieves near-optimal performance for the Fair-AoI objective because it provides balanced access opportunities across devices.
However, its performance degrades under the Max-rate objective, since equal transmission allocation does not account for differences in the achievable data rates of individual devices.
In contrast, our proposed method, along with end-to-end training, learns to prioritize devices with more favorable channel conditions, resulting in performance closer to the optimal.

Lastly, we provide a summary of performance comparison in Table~\ref{tab:performance}.
We observe that the algorithms achieve comparable final rate and AoI performance.
However, the proposed FM-aided MARL approach arrives at its $95\%$-optimal performance level {\em substantially faster} than the end-to-end training approach.

\begin{table}[!t]
    \centering
    \caption{Performance Comparison of End-to-End (E-E) and FM-aided MARL for $K=4$ (top) and $K=8$ (bottom). AoI is measured for the fair-AoI task, and rate is measured for both max-rate and fair-rate tasks.}
    \setlength{\tabcolsep}{2.5pt}
    \begin{tabular}{cc|cccc|c}
        \hline
        Task & Alg. & Mean & Min. & Max. & N-Gap & 95\%-T (s) \\
        \hline
        Fair-AoI & E-E & 0.731 $\pm$ 0.010 & 0.726 & 0.734 & 0.011 & 335 \\
        (ms) & FM & 0.769 $\pm$ 0.032 & 0.757 & 0.772 & 0.019 & 108 \\
        \hline
        Max-rate & E-E & 17.044 $\pm$ 0.266 & 4.092 & 41.225 & 0.901 & 380\\
        (Mbps) & FM & 17.924 $\pm$ 0.345 & 3.918 & 43.865 & 0.911  & 159 \\
        \hline
        Fair-rate & E-E & 9.049 $\pm$ 0.186 & 6.903 & 11.956 & 0.423 & 373 \\
        (Mbps) & FM & 8.083 $\pm$ 0.131 & 6.143 & 11.590 & 0.463 & 104 \\
        \hline
    \end{tabular}
    \label{tab:performance}
    \vspace{-1mm}
\end{table}

\begin{table}[!t]
    \centering
    \setlength{\tabcolsep}{2.5pt}
    \begin{tabular}{cc|cccc|c}
        \hline
        Task & Alg. & Mean & Min. & Max. & N-Gap & 95\%-T (s) \\
        \hline
        Fair AoI & E-E & 0.366 $\pm$ 0.014 & 0.364 & 0.368 & 0.003 & 723 \\
        (ms) & FM & 0.368 $\pm$ 0.025 & 0.365 & 0.370 & 0.013 & 231 \\
        \hline
        Max Rate & E-E & 7.235 $\pm$ 0.375 & 3.895 & 30.175 & 0.871 & 812\\
        (Mbps) & FM & 7.152 $\pm$ 0.236 & 4.100 & 29.685 & 0.862  & 345 \\
        \hline
        Fair Rate & E-E & 6.365 $\pm$ 0.352 & 5.851 & 6.864 & 0.148 & 768 \\
        (Mbps) & FM & 6.335 $\pm$ 0.244 & 5.642 & 7.025 & 0.197 & 225 \\
        \hline
    \end{tabular}
    \label{tab:performance2}
    \vspace{-0.1in}
\end{table}

\section{Conclusion}\label{sec:conclusion}

In this work, we proposed an FM-aided decentralized MARL framework for RA network optimization.
By employing an FM as a reward-agnostic critic backbone, our algorithm effectively leverages self-supervised knowledge on network dynamics and improves learning efficiency.
Compared with end-to-end training, the FM-aided approach achieves faster convergence, demonstrating its effectiveness in optimizing throughput and AoI across diverse objectives.
Future work includes extending this framework to more complex and sophisticated wireless RA scenarios with larger action spaces (e.g., considering transmit power control).

\begin{acks}
    This work was supported in part by the National Science Foundation (NSF) under Grant CNS2110259, Grant CNS2112471, and Grant ECCS233104; in part by the Office of Naval Research (ONR) under Grant N000142412729; in part by the Defense Advanced Research Projects Agency (DARPA) under Grant D24AP00265, Grant HR00112520019, and Grant HR0011261E001; and in part by Meta Research Award INB3267726.
\end{acks}

\bibliographystyle{ACM-Reference-Format}
\bibliography{mybib}

\end{document}